\documentclass{article}

 \usepackage[preprint]{neurips_2026}

\usepackage[utf8]{inputenc} 
\usepackage[T1]{fontenc}    
\usepackage{hyperref}       
\usepackage{url}            
\usepackage{booktabs}       
\usepackage{amsfonts}       
\usepackage{nicefrac}       
\usepackage{microtype}      
\usepackage{xcolor}         

\usepackage{graphicx}
\usepackage{subcaption}
\usepackage{placeins}
\usepackage{xurl}

\usepackage{amsmath}
\usepackage{amssymb}
\usepackage{mathtools}
\usepackage{amsthm}
\usepackage{tabularx}

\usepackage{makecell}
\usepackage{siunitx}
\usepackage{array}
\usepackage{enumitem}

\usepackage{wrapfig}
\usepackage{placeins}
\usepackage{multirow}

\definecolor{darkblue}{rgb}{0.0, 0.17, 0.58}
\hypersetup{colorlinks,breaklinks,linkcolor=red,citecolor=darkblue}
\usepackage{titletoc}

\newtheorem{proposition}{Proposition}

\newcolumntype{Y}{>{\raggedright\arraybackslash}X}

\title{Entropy Pacing Policy Optimization for Multi-Task Agentic
Reinforcement Learning}

\author{
 \textbf{Zetian Hu\textsuperscript{1}\textsuperscript{*}},
 \textbf{Shunyu Liu\textsuperscript{1}},
 \textbf{Junjie Zhang\textsuperscript{1}},
 \textbf{Yongcheng Jing\textsuperscript{1}},
 \textbf{Ting-En Lin\textsuperscript{2}},
\\
 \textbf{Yongbin Li\textsuperscript{2}},
 \textbf{Dacheng Tao\textsuperscript{1}\textsuperscript{\textdagger}}
\\
\\
 \textsuperscript{1}Generative AI Lab, College of Computing and Data Science,
 \\
 Nanyang Technological University, Singapore 639798
\\
 \textsuperscript{2}Tongyi Lab, Alibaba Group
}

\begin{document}

\maketitle

\begingroup
\renewcommand\thefootnote{\fnsymbol{footnote}}
\footnotetext[1]{Email: zetian.hu@ntu.edu.sg}
\footnotetext[2]{Corresponding authors: dacheng.tao@gmail.com}
\endgroup

\begin{abstract}
Recent breakthroughs of Reinforcement Learning~(RL) have highlighted its potential for complex agentic Large Language Model~(LLM) tasks. However, existing efforts largely focus on single-task settings, whereas real-world deployment necessitates a generalist agent capable of solving multiple tasks simultaneously.
In this work, we identify a critical yet underexplored phenomenon in multi-task agentic RL: \textit{different tasks can exhibit exploration-exploitation pace mismatch}. 
Specifically, easier tasks may converge early to low-entropy policies that hinder learning on harder tasks, while harder tasks can, in turn, push easier tasks back toward high-entropy exploration. 
This back-and-forth interaction creates inter-task entropy crossovers and frequent entropy spikes.
Inspired by this observation, we introduce Entropy Pacing Policy Optimization~(EPPO) for multi-task agentic LLMs, which coordinates entropy across tasks to stabilize multi-task optimization.
At the core of EPPO is a task-wise dynamic clipping mechanism that replaces the fixed clipping threshold in Group Relative Policy Optimization~(GRPO) with a task entropy-aware adaptive bound, tightening updates for over-confident tasks while relaxing them for under-explored ones.
Experiments on the multi-task agentic benchmarks demonstrate that the proposed EPPO yields results superior to its counterparts. 

\end{abstract}

\section{Introduction}
Recent advances in large language models (LLMs) have moved post-training beyond high-quality single responses toward interactive agents that act over long horizons and use tools to accomplish goals in diverse environments~\citep{plaat2025agentic,wang2025toward,luo2025large}.
Compared with single-turn tasks such as math problem solving~\citep{guo2025deepseek} or code generation~\citep{guo2024deepseek}, agentic settings involve sparse, delayed feedback and thus naturally call for Reinforcement Learning~(RL), where a policy must plan, explore, recover from errors, and adapt over multiple steps~\citep{zhang2025landscape,zhang2026prorl}.
RL-trained LLM agents have demonstrated strong performance across diverse tasks, including search and research~\citep{feng2025retool,jin2025search}, code generation~\citep{zhang2024o1}, and GUI navigation~\citep{liu2025infigui}.
However, most existing work optimizes agents for a single task or domain~\citep{wu2025webdancer,li2025webthinker}, which limits transfer and complicates real-world deployment. This motivates multi-task agentic RL: training one unified policy that can generalize across heterogeneous agentic tasks.

Scaling multi-task agentic RL remains difficult. Tasks of different difficulty are highly heterogeneous, and agentic environments vary widely in horizon length, feedback density, and tool-use structure, often causing imbalance and negative transfer in a shared policy~\citep{zhang2025agentrl}.
Existing multi-task approaches roughly fall into two families. 
(i) Expert-training and distillation methods such as MiMo-V2-Flash~\citep{xiao2026mimo} and DeepSeek-V3.2~\citep{liu2025deepseek} first train domain specialists, then distill or integrate their expertise into a unified model, which alleviates capability imbalance and improves knowledge integration across domains. 
(ii) End-to-end joint training methods train a single policy directly with multi-task balancing and scalable environments~\citep{zhang2025divide,li2025internbootcamp,dai2024cotbal}.
Among them, AgentRL~\citep{zhang2025agentrl} makes an important step for multi-task agentic RL by providing an efficient asynchronous training pipeline and a scalable environment stack, enabling large-scale heterogeneous training. It introduces task advantage normalization, which normalizes token-level advantages within each task batch to reduce inter-task variance and stabilize joint optimization. 

Despite these promising results, we observe a critical yet underexplored phenomenon during multi-task agentic GRPO training: tasks can drift into exploration-exploitation pace mismatch. Concretely, easier tasks may collapse early to low-entropy policies, which suppresses exploration for harder tasks; conversely, gradients from still-exploratory tasks can perturb the shared policy and destabilize already-converged tasks.
Empirically, this manifests as frequent inter-task entropy crossovers over time, together with late-stage entropy spikes that indicate instability. 
While entropy bonuses or KL penalties may partially stabilize entropy dynamics~\citep{cui2025entropy,jiang2025rethinking}, a single global knob is unlikely to fit all tasks simultaneously in the multi-task setting~\citep{yang2025dcpo}.
We attribute a key part of this failure mode to GRPO's global clipping range. As a task-agnostic trust-region proxy, one shared clipping constraint cannot simultaneously slow down over-confident tasks to preserve useful exploration and allow sufficiently strong updates for under-explored tasks. This mismatch amplifies coupling through shared parameters, producing the observed entropy crossings and spikes.

To address this pace mismatch, we propose \textbf{Entropy Pacing Policy Optimization (EPPO)}, which replaces GRPO's global fixed clipping range with a task-wise dynamic clip mechanism guided by each task's relative entropy pace. EPPO treats policy entropy as a proxy for exploration-exploitation state: (i) it tracks per-task entropy, anchors an initial reference, and computes an entropy collapse ratio to make pace comparable across heterogeneous tasks; (ii) it maps the standardized entropy collapse ratio of each task to a bounded clip adjustment via, enlarging clip range for under-explored tasks and shrinking it for over-confident ones to balance learning pace; (iii) it adds a stability-aware trend constraint that downscales clip range when entropy rises, damping oscillations and preventing abrupt policy shifts. EPPO introduces lightweight bookkeeping and plugs into the existing multi-task agentic GRPO by replacing the fixed clip with the task-wise dynamic clip mechanism.

Our contributions can be summarized as follows:
\begin{itemize}[leftmargin=*]
  \item We identify and characterize the pace mismatch phenomenon in multi-task agentic GRPO. It is an important yet underexplored limitation of shared-parameter training, where tasks exhibit divergent entropy-collapse dynamics, including inter-task entropy crossovers and frequent entropy spikes.
  \item We propose EPPO, a task-wise dynamic clipping mechanism that replaces GRPO’s fixed clipping range with an entropy-aware adaptive bound, coordinating exploration across tasks in multi-task agentic reinforcement learning.
  \item Experiments on agentic benchmarks show that EPPO improves average success rate over multi-task baselines and stabilizes training by reducing entropy crossovers and suppressing late-stage spikes.
\end{itemize}

\section{Related Work}

\textbf{Reinforcement Learning for Generalist Agents.}
With the rapid evolution of LLMs, research has moved from static, single-turn question answering toward dynamic, multi-turn generalist agents~\citep{plaat2025agentic,wang2025toward,luo2025large}. RL has been widely used to improve agentic capabilities in interactive environments~\citep{jin2025search,feng2025retool,li2025torl,modecrua2026multi}, but most efforts remain single-task or domain-specific, lacking the integration needed by generalist agents.

Prior work on generalist agents mainly follows two paradigms. Expert-training and distillation methods train domain specialists and then merge their capabilities into a unified model, as in MiMo-V2-Flash~\citep{xiao2026mimo} and DeepSeek-V3.2~\citep{liu2025deepseek}; however, this staged pipeline adds complexity, depends on teacher performance, and may underuse positive transfer. End-to-end multi-task RL instead optimizes a shared policy across heterogeneous environments~\citep{zhang2025divide,li2025internbootcamp}, with AgentRL~\citep{zhang2025agentrl} providing scalable infrastructure for large-scale joint training. Yet learning a unified policy remains difficult due to large variations in task difficulty and dynamics.

\textbf{Multi-Task Optimization Reinforcement Learning.}
A key bottleneck in multi-task RL is gradient conflict, where updates for one task can harm others~\citep{fernando2023mitigating}. Existing methods mitigate this through gradient manipulation, such as PCGrad~\citep{yu2020gradient} and CAGrad~\citep{liu2021conflict}, or through task-level normalization and weighting. AgentRL standardizes advantage scales across tasks via task advantage normalization~\citep{zhang2025agentrl}; CoTBal~\citep{dai2024cotbal} balances inter-task contribution and intra-task difficulty in multi-task visual instruction tuning. However, these methods mainly address gradient magnitude or direction, while often overlooking temporal disparities in learning dynamics, where simple tasks may converge early, and complex tasks remain under-explored.

\textbf{Entropy Dynamics and Adaptive Clipping.}
Policy entropy is a critical indicator of exploration capacity~\citep{cheng2025reasoning}. Recent studies show that entropy reduction is closely related to performance gains but may also cause premature entropy collapse~\citep{cui2025entropy}; in large action spaces, naive entropy regularization can further fail or destabilize training~\citep{jiang2025rethinking}. Since standard GRPO uses a static clipping range, recent work has explored adaptive clipping: DCPO adjusts token-level clipping bounds to encourage exploration of rare tokens~\citep{yang2025dcpo}, while BAPO adapts clipping bounds to balance positive and negative advantage updates and preserve entropy in off-policy RL~\citep{xi2025bapo}. Nevertheless, existing adaptive clipping methods mainly operate at the single-task level and do not explicitly handle inter-task entropy conflicts under shared-policy multi-task optimization.

\section{Preliminaries}
\label{sec:prelim}

\subsection{Multi-Task Reinforcement Learning}
We consider a set of $N$ tasks $\mathcal{T}=\{1,\dots,N\}$, where each task $i$ is an episodic Markov decision process $\mathcal{M}_i=\langle\mathcal{S},\mathcal{A},P_i,r_i,\gamma\rangle$. Here, $\mathcal{S}$ denotes the state space, $\mathcal{A}$ the action space, $P_i(s' \mid s,a)$ the task-specific transition dynamics, $r_i:\mathcal{S}\times\mathcal{A}\rightarrow\mathbb{R}$ the task-specific reward function, and $\gamma\in(0,1)$ the discount factor. A shared LLM policy $\pi_\theta(a | s, i)$, parameterized by $\theta$, is trained across tasks, optionally conditioned on a task identifier or environment context. The multi-task objective maximizes the expected returns:
\begin{equation}
J(\theta)=\sum_{i=1}^{N} \, \mathbb{E}_{\tau\sim \pi_\theta,\,P_i}\Big[\sum_{t=0}^{T_i-1}\gamma^t r_i(s_t,a_t)\Big].
\end{equation}
In agentic LLM settings, trajectories are multi-turn, and tasks are inherently heterogeneous, making exploration dynamics task-dependent.

\subsection{Group Relative Policy Optimization (GRPO)}
GRPO is a critic-free policy optimization method that replaces value-function baselines with group-relative reward normalization. For each query/prompt $q$, GRPO samples a group of $G$ outputs $\{o^{(g)}\}_{g=1}^{G}$ from an old policy $\pi_{\theta_{\text{old}}}$ and assigns each output a scalar reward $r^{(g)}$ provided by a verifier or a reward model. It then constructs a group-relative advantage using only within-group statistics:
\begin{equation}
\begin{split}
\hat A^{(g)}
= \frac{r^{(g)}-\mu_r}{\sigma_r+\varepsilon}, 
 \mu_r = \frac{1}{G}\sum_{g=1}^{G} r^{(g)},
\sigma_r
= \sqrt{\frac{1}{G}\sum_{g=1}^{G}\bigl(r^{(g)}-\mu_r\bigr)^2}.
\end{split}
\end{equation}

GRPO performs a token-level update with clipped importance ratios and typically includes a KL regularizer to a frozen reference policy $\pi_{\text{ref}}$:
\begin{equation}
\begin{split}
\max_{\theta}\;
\mathbb{E}\Bigg[
\frac{1}{G}\sum_{g=1}^{G}\frac{1}{|o^{(g)}|}\sum_{t=1}^{|o^{(g)}|}
\min\Big(\rho_{g,t}(\theta)\hat A^{(g)},\ \text{clip}(\rho_{g,t}(\theta),1-\epsilon,1+\epsilon)\hat A^{(g)}\Big)
\;-\;\lambda_{KL} D_{\text{KL}}\big(\pi_\theta\;\|\;\pi_{\text{ref}}\big)
\Bigg],
\end{split}
\end{equation}
where $\rho_{g,t}(\theta)=\pi_\theta(o^{(g)}_t\mid q,o^{(g)}_{<t})/\pi_{\theta_{\text{old}}}(o^{(g)}_t\mid q,o^{(g)}_{<t})$. By estimating the baseline from group scores, GRPO avoids training a critic while retaining PPO-like stability.

\subsection{Task Advantage Normalization.}
In multi-task agentic RL, task heterogeneity often leads to imbalanced gradient contributions under joint training. AgentRL~\citep{zhang2025agentrl} introduces task advantage normalization to stabilize multi-task optimization by normalizing advantages within each task. For task $i$, let $\mathcal{A}_i$ denote the set of all token-level advantage estimates in the current batch; each advantage is normalized as
\begin{equation}
\tilde{A}_i = \frac{A_i - \mu_i}{\sigma_i + \epsilon},
\end{equation}
where $\mu_i$ and $\sigma_i$ are the mean and standard deviation of $\mathcal{A}_i$,
respectively. 
This task-wise normalization enforces zero mean and unit variance for each task’s advantages, reducing inter-task variance and preventing tasks with larger advantage magnitudes from dominating updates. However, while it equalizes gradient scale across tasks, it does not account for differences in task learning dynamics, such as varying degrees of exploration or policy maturity, which motivates our subsequent extension.

\begin{figure*}[!t]
    \centering

    \includegraphics[width=0.8\textwidth]{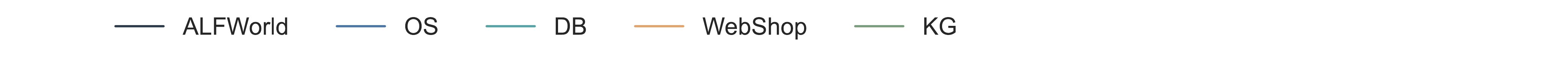}
    \vspace{-1em}

    \begin{subfigure}[t]{0.34\textwidth}
        \centering
        \includegraphics[width=\linewidth]{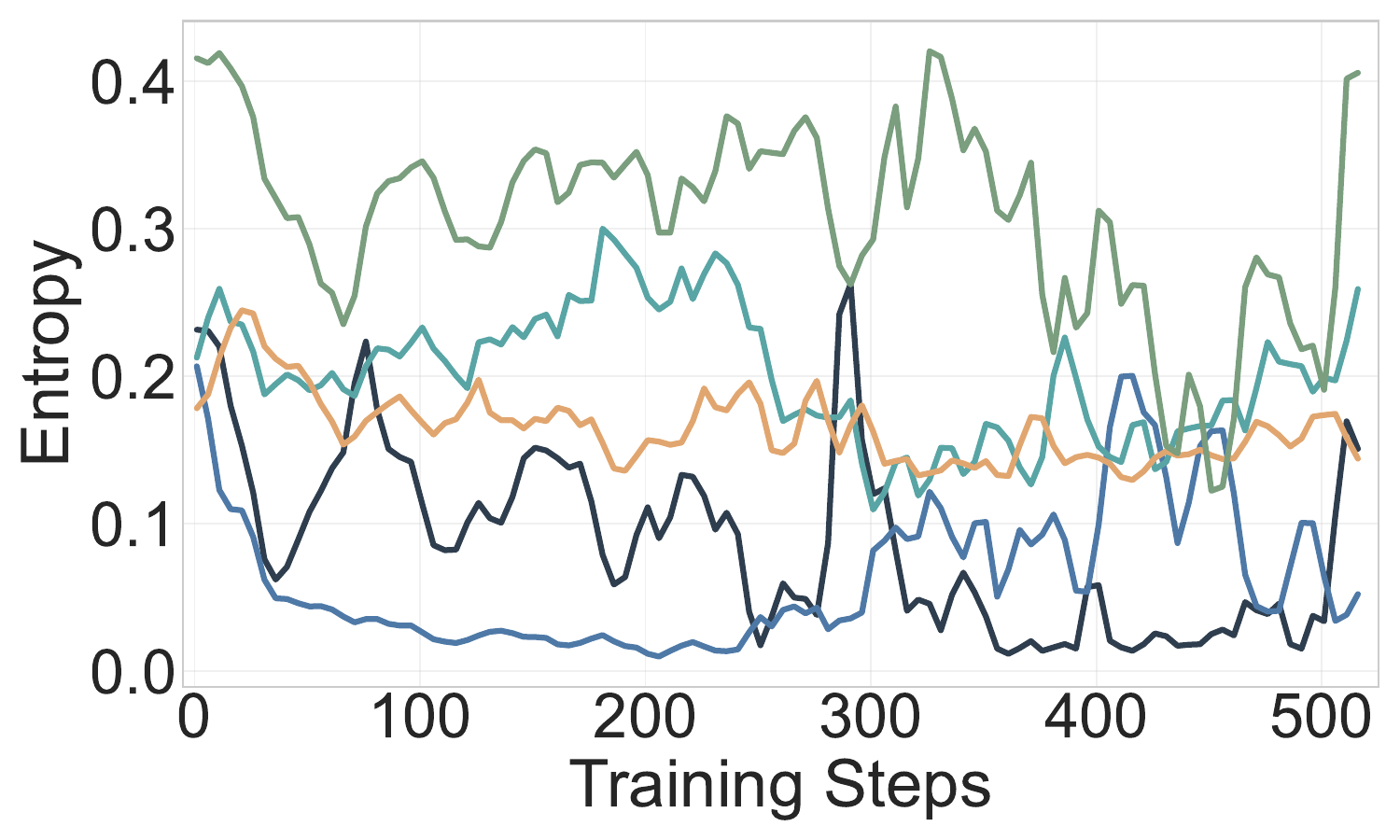}
        \caption{}
        \label{fig:entropy_baseline}
    \end{subfigure}
    \begin{subfigure}[t]{0.34\textwidth}
        \centering
        \includegraphics[width=\linewidth]{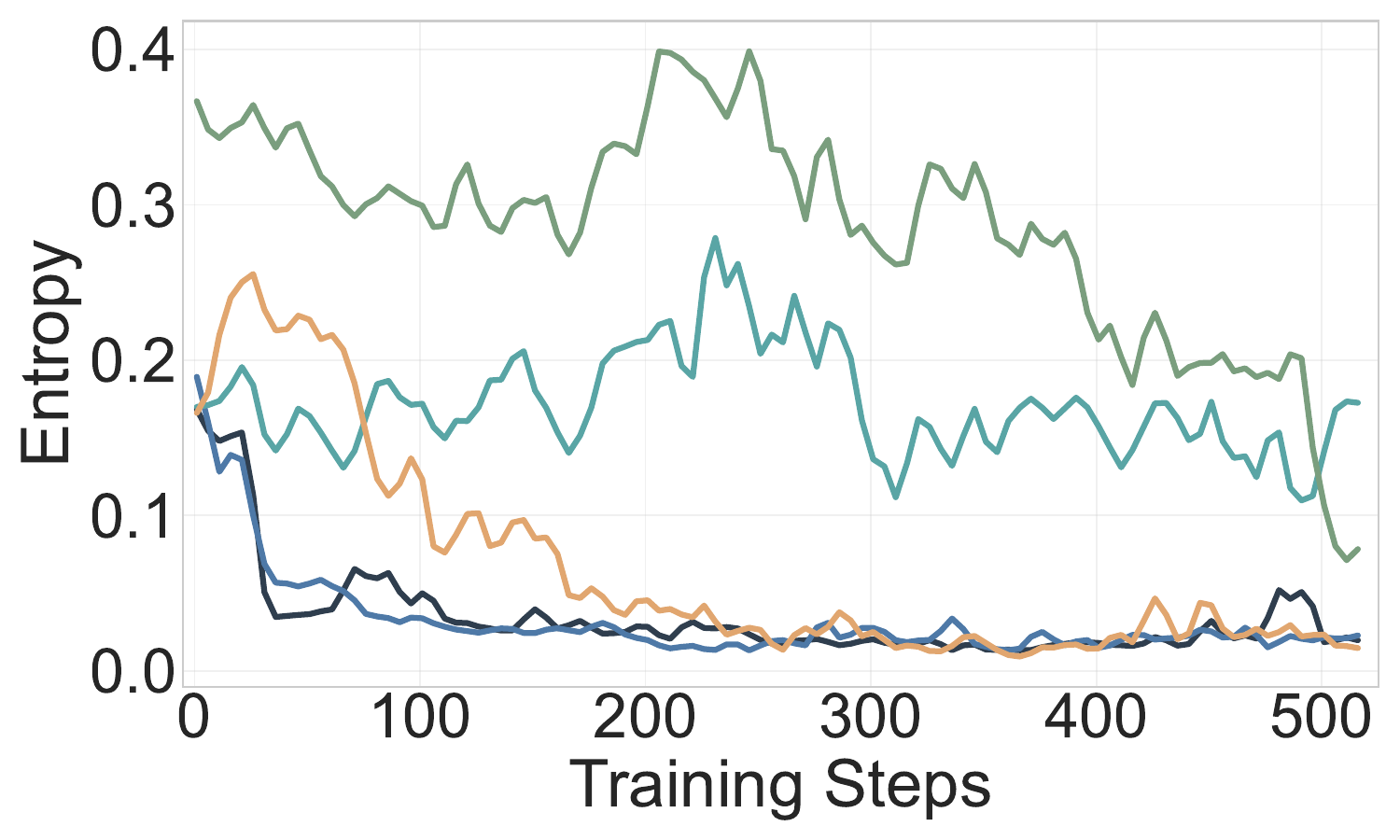}
        \caption{}
        \label{fig:entropy_our}
    \end{subfigure}
    \begin{subfigure}[t]{0.3\textwidth}
        \centering
        \includegraphics[width=\linewidth]{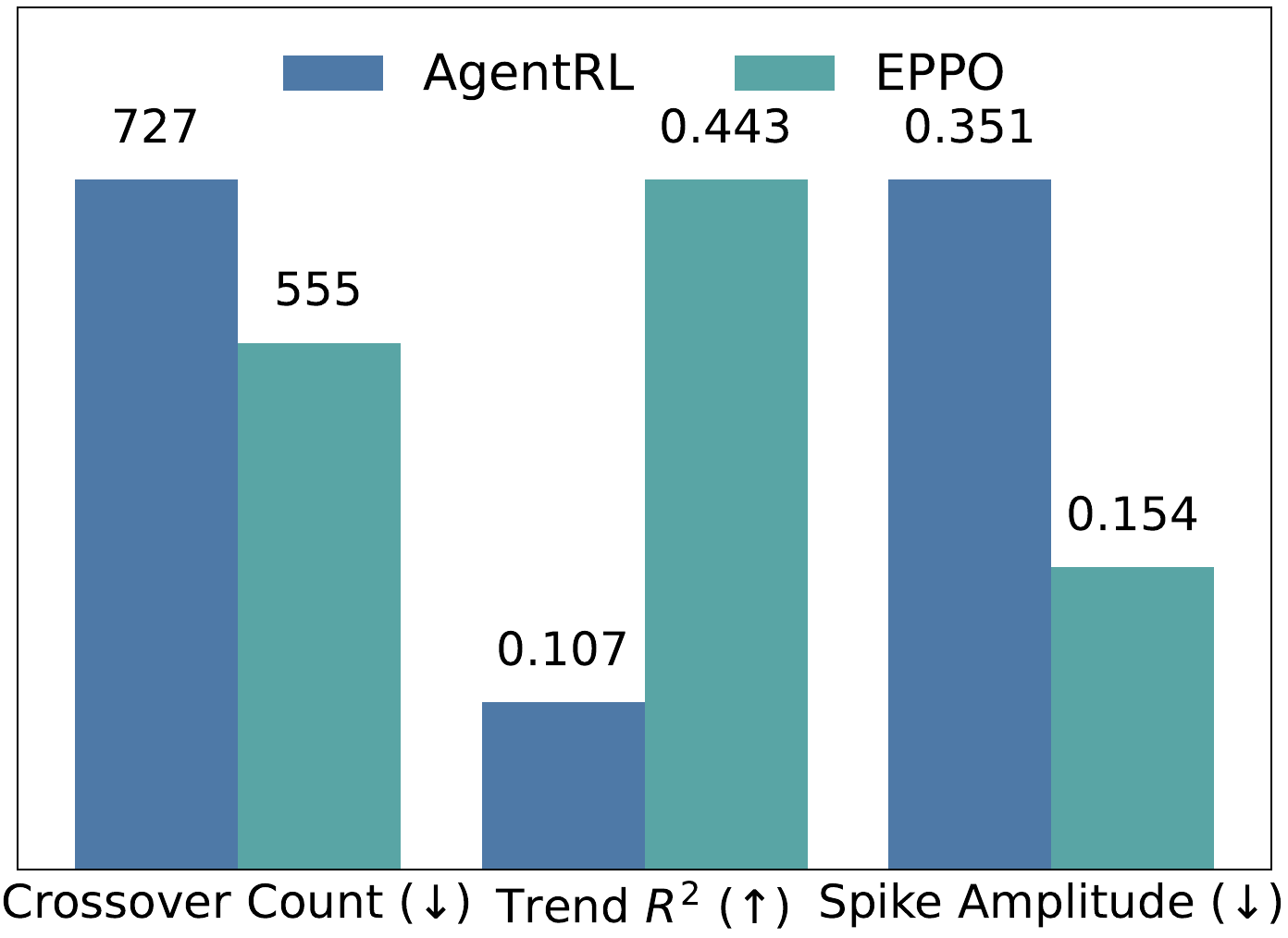}
        \caption{}
        \label{fig:entropy_metric}
    \end{subfigure}
    \caption{Task-wise policy entropy dynamics during multi-task agentic training. (a) The AgentRL exhibits pace mismatch, characterized by asynchronous entropy collapse with frequent inter-task crossovers and late-stage entropy spikes. (b) EPPO organizes task entropies into stable, persistent bands with fewer crossovers and suppressed spikes. (c) Quantitative diagnostics of entropy instability: crossover count sums pairwise sign-flip events (lower is better); trend $R^2$ is the average coefficient of determination from a linear fit of task entropy versus $\log t$ (higher is better); late spike amplitude is the maximum upward entropy jump in the last 30\% of training (lower is better).}
    \label{fig:entropy}
\end{figure*}

\section{Method}
\label{sec:method}

We identify a pace mismatch in multi-task agentic GRPO caused by divergent entropy collapse dynamics across tasks. Motivated by this, we propose a task-wise dynamic clipping mechanism that replaces GRPO’s fixed clipping range with an entropy-aware adaptive bound.

\subsection{Pace Mismatch Phenomenon}
\label{sec:phenomenon}
Under joint training with a shared policy, heterogeneous agentic tasks exhibit markedly asynchronous exploration-exploitation dynamics. Figure~\ref{fig:entropy_baseline} and Figure~\ref{fig:entropy_our} illustrate the task-wise entropy trajectories during AgentRL and EPPO training, respectively, while Figure~\ref{fig:entropy_metric} quantifies their stability using three diagnostics:
\textbf{(i) Crossover count}, which measures how often the relative entropy ordering between tasks flips and captures rank-order instability under shared-parameter coupling;
\textbf{(ii) Trend $R^2$}, which measures how well each task’s entropy follows a linear trend versus $\log t$ and serves as a proxy for trend consistency;
and \textbf{(iii) Late spike amplitude}, which measures the largest upward entropy jump in the last 30\% of training and captures late-stage destabilization.

Figure~\ref{fig:entropy_baseline} shows a clear pace mismatch: some tasks rapidly collapse to near-deterministic behavior early in training, while others remain high-entropy for much longer, indicating continued search for effective strategies. Despite sharing the same model parameters and update rule, tasks therefore occupy substantially different learning phases. This mismatch is further reflected by frequent inter-task crossings, suggesting strong coupling through shared parameters. Updates dominated by over-confident tasks can sharpen the shared policy and push other tasks toward premature exploitation, whereas updates from under-explored tasks can introduce distributional shifts that perturb tasks that have already stabilized.

A further symptom is the emergence of sporadic late-stage entropy spikes. Once some tasks have collapsed, corrective updates required by under-explored tasks may perturb the shared policy, destabilizing previously stable tasks and producing transient re-randomization. Together, the rapid entropy collapse, frequent crossovers, and late spikes indicate that a single shared trust-region constraint struggles to preserve task-specific exploration-exploitation states.

We identify GRPO’s global clipping range $\epsilon$ as a key contributor to this phenomenon. A single task-agnostic trust-region proxy cannot simultaneously slow down tasks that collapse too early to preserve useful exploration and allow sufficiently aggressive updates for tasks that remain under-explored. This motivates a pace-aware update mechanism that adapts the effective update constraint per task, using exploration state as a control signal. More detailed analyses are in Section~\ref{sec:analysis}.

\subsection{Entropy Pacing Policy Optimization}
\label{sec:eppo}

We propose Entropy Pacing Policy Optimization, which replaces GRPO's global clipping range with a task-wise dynamic clip $\epsilon_i(t)$ driven by each task's relative entropy pace. The key design choice is to treat entropy as an online indicator of a task's position on the exploration-exploitation trajectory, and to synchronize learning pace across tasks by modulating how aggressively each task can update the shared policy. 
We do not assume that entropy is equivalent to learning progress, nor that the clipping range directly represents exploration. Instead, entropy is used only as an online indicator of a task’s exploration-exploitation state, while $\epsilon_i(t)$ regulates the task-wise update pace by controlling how far the policy can move.
We provide a theoretical analysis of EPPO in Appendix~\ref{app:theory}.

As illustrated in Figure~\ref{fig:eppo_overview}, EPPO follows a modular pipeline with (A) an entropy tracker, (B) a global progress pacer, and (C) a stability-aware trend constraint, which jointly generate the
per-task clipping range $\epsilon_i(t)$ for the multi-task update engine.

\begin{figure*}[!t]
  \centering
  \includegraphics[width=1\textwidth]{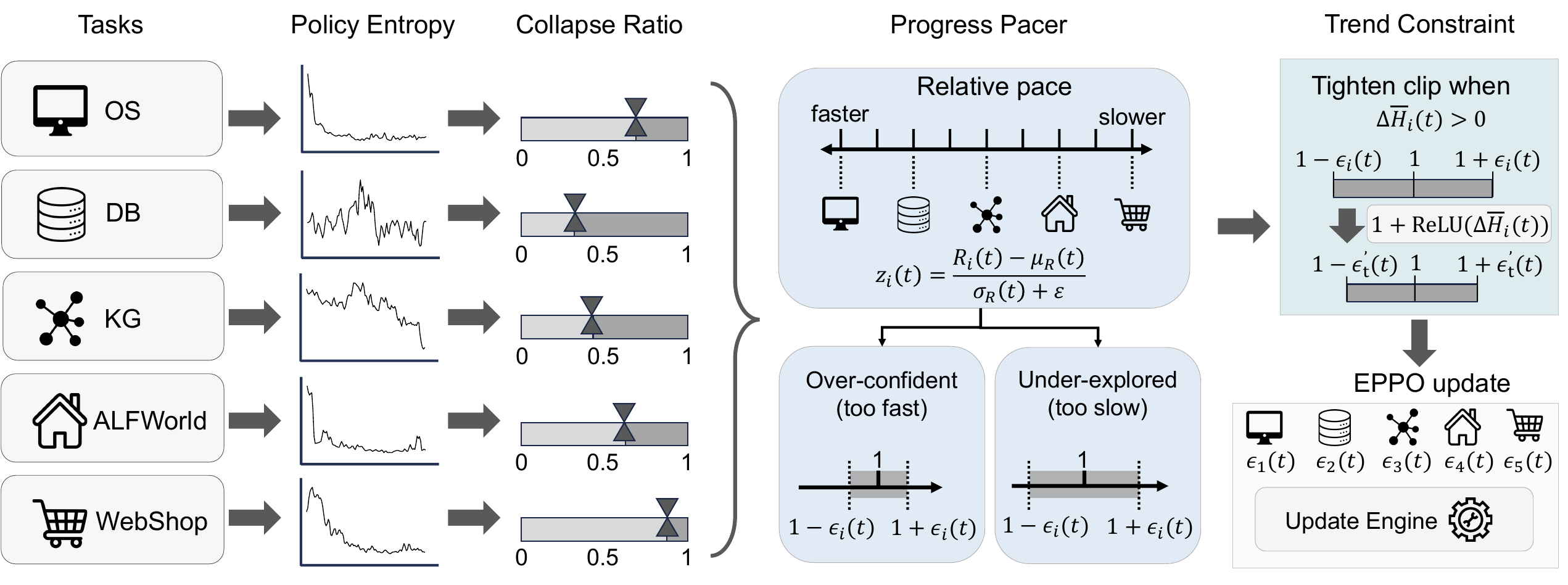}
  \caption{EPPO overview. Per-task entropy is tracked with EMA and converted to a normalized collapse ratio; a cohort-level pacer maps relative pace to task-wise dynamic clipping $\epsilon_i(t)$, while a trend constraint downscales updates when entropy rises to improve stability.}
  \label{fig:eppo_overview}
\end{figure*}

\textbf{Entropy tracker.}
We first build a stable, task-comparable pace signal from entropy, mitigating multi-turn noise across tasks.
For task $i$, we measure exploration via policy entropy:
\begin{equation}
H_i(t)=\mathbb{E}_{(s_t,i)}\left[-\sum_a \pi_\theta(a\mid s_t,i)\log \pi_\theta(a\mid s_t,i)\right],
\end{equation}
and maintain a smoothed estimate via an exponential moving average (EMA):
\begin{equation}
\bar H_i(t)=\beta \bar H_i(t-1) + (1-\beta)H_i(t).
\end{equation}
We store an initial entropy reference $H_i(0)$ and define the entropy collapse ratio
\begin{equation}
R_i(t)=\frac{\bar H_i(t)}{H_i(0)},
\end{equation}
so that $R_i(t)$ is dimensionless and directly comparable across tasks with different entropy scales.

\textbf{Progress pacer.}
We then convert relative pace differences into a task-specific clipping range to assign bounded update budgets across tasks. At update $t$, we standardize $R_i(t)$ using the mean $\mu_R(t)$ and standard deviation $\sigma_R(t)$ over $\{R_j(t)\}_{j=1}^N$ into task-wise $z$-score:
\begin{equation}\label{eq:z_score}
z_i(t)=\frac{R_i(t)-\mu_R(t)}{\sigma_R(t)+\varepsilon}.
\end{equation}
Using the cohort-normalized $z$-score makes pacing depend on relative progress rather than absolute entropy. 
We compute a bounded scaling factor and obtain the per-task clipping range:
\begin{equation}
s_i(t)=1+\alpha \tanh\big(z_i(t)\big), \qquad
\epsilon_i(t)=\epsilon_0 \cdot s_i(t),
\end{equation}
where $\epsilon_0$ is the base clipping range and $\alpha\in(0,1)$ controls the pacing strength. 
Since $\tanh(\cdot)\in[-1,1]$, the effective scaling is bounded as $s_i(t)\in[1-\alpha,\,1+\alpha]$, preventing excessively large or vanishing update constraints. This implements pacing: tasks with over-confident entropy collapse, indicated by smaller $R_i(t)$ and $z_i(t)<0$, receive $\epsilon_i(t)<\epsilon_0$ for more conservative updates, while under-explored tasks ($z_i(t)>0$) receive $\epsilon_i(t)>\epsilon_0$ to accelerate learning.

\textbf{Stability trend constraint.}
Finally, to suppress interference-induced oscillations, we detect the entropy trend and tighten the trust region. 
We compute the entropy trend by:
\begin{equation}
\Delta \bar H_i(t)=\bar H_i(t)-\bar H_i(t-1),
\end{equation}
and penalize the clip range when entropy increases:
\begin{equation}
\epsilon_i(t)\leftarrow \frac{\epsilon_i(t)}{1+\,\text{ReLU}\big(\Delta \bar H_i(t)\big)}.
\end{equation}
Intuitively, when a task enters a regime of increasing randomness, reducing $\epsilon_i(t)$ dampens abrupt policy shifts and discourages oscillatory behavior without removing the relative pacing effect.

Finally, EPPO replaces the global clip in GRPO with the task-wise clip:
\begin{equation}
\begin{split}
\mathcal{L}^{\text{EPPO}}_i(\theta)=
\mathbb{E}\Big[
\min\big(\rho_{g,t}(\theta)A_t^i,\ \text{clip}(\rho_{g,t}(\theta), 1-\epsilon_i(t),1+\epsilon_i(t))\,A_t^i\big)
\Big].
\end{split}
\end{equation} 
EPPO needs only lightweight bookkeeping for EMA and task-level statistics, and can be integrated into existing multi-task agentic GRPO pipelines by replacing the fixed clip range with $\epsilon_i(t)$. 

\section{Experiments}
\subsection{Benchmarks \& Setup}
\textbf{Benchmarks.}
We follow AgentRL~\citep{zhang2025agentrl} and evaluate on five multi-turn, interactive agentic tasks from AgentBench~\citep{liu2023agentbench}: Operating System~(OS), Database~(DB), Knowledge Graph~(KG), ALFWorld~\citep{shridhar2020alfworld}, and WebShop~\citep{yao2022webshop}, covering operating-system command execution, database querying, knowledge-graph navigation, text-based household manipulation, and web shopping, respectively. All environments are converted to a unified function-call interaction format, with standardized controller APIs for multi-turn interaction. We adopt the same data construction strategy as AgentRL: official datasets for ALFWorld and WebShop, and Self-Instruct–style synthetic data for OS, KG, and DB.

\textbf{Baselines.}
For training-free comparison, we include the baselines explicitly reported in AgentRL, covering representative closed-source API models (e.g., Claude-Sonnet~\citep{anthropic_claude_sonnet4_2025}, GPT-5~\citep{singh2025openai}, and OpenAI o-series~\citep{openai_o3_o4_mini_systemcard_2025}), strong open-source instruction-tuned models (Qwen2.5-Instruct series~\citep{qwen2024qwen25technicalreport}, DeepSeek-V3~\citep{liu2024deepseek}, and DeepSeek-R1~\citep{guo2025deepseek}), as well as prior agent training methods evaluated on AgentBench, such as Hephaestus~\citep{zhuang2025hephaestus} and AgentLM~\citep{zeng2024agenttuning}. For RL training comparison, we use AgentRL~\citep{zhang2025agentrl} as the primary baseline and further compare against adaptive clipping methods, including DCPO~\citep{yang2025dcpo} and BAPO~\citep{xi2025bapo}.

\textbf{Implementation Details.}
We adopt Qwen2.5-3B-Instruct as our backbone policy, conducting training on the AgentRL infrastructure on 8 NVIDIA H20 GPUs. For EPPO, we initialize a base clipping range of $\epsilon_0=0.2$ and set the pacing strength $\alpha=0.4$; the entropy EMA decay is set to $\beta=0.95$.  We report task success rate and aggregate results as the mean over four evaluations with standard deviation. More implementation details and hyperparameter settings are in Appendix~\ref{app:implementation}.

\subsection{Main Results}

\begin{table}[!t]
\centering
\caption{Main results on five-task multi-task training. We report the average success rate over four evaluations. Our reproduced results for AgentRL differ slightly from those reported in the original paper; a detailed analysis of the discrepancy is provided in Appendix~\ref{app:discrepancies}. ALF denotes ALFWorld.}
\vspace{1em}
\label{tab:main_results_5}
\begin{tabularx}{\linewidth}{Xcccccc}
\toprule
\textbf{Model} & \textbf{ALF} & \textbf{DB} & \textbf{KG} & \textbf{OS} & \textbf{WebShop} & \textbf{AVG} \\
\midrule
\multicolumn{7}{c}{\textit{API LLMs}} \\
\midrule
Claude-Sonnet-4 & 73.6 & 70.1 & 63.4 & 45.3 & 34.6 & 57.4 \\
Claude-Sonnet-4 Thinking & 69.0 & 68.4 & 64.4 & 51.0 & 38.3 & 58.2 \\
o3-mini & 28.4 & 56.5 & 51.8 & 35.1 & 32.7 & 40.9 \\
o4-mini & 32.6 & 63.4 & 32.4 & 41.8 & 28.5 & 39.7 \\
GPT-5 & 65.4 & 63.2 & 64.1 & 34.5 & 33.7 & 52.2 \\
\midrule
\multicolumn{7}{c}{\textit{Open LLMs}} \\
\midrule
DeepSeek-V3 & 31.9 & 58.4 & 14.0 & 53.0 & 23.4 & 36.1 \\
DeepSeek-R1 & 51.4 & 60.4 & 50.2 & 53.6 & 31.0 & 49.3 \\
Qwen2.5-3B-Instruct & 2.2 & 25.5 & 4.8 & 23.9 & 5.3 & 12.3 \\
Qwen2.5-14B-Instruct & 8.7 & 48.4 & 35.3 & 26.0 & 17.6 & 27.2 \\
Qwen2.5-32B-Instruct & 32.1 & 55.8 & 33.8 & 37.0 & 27.5 & 37.2 \\
Hephaestus-8B-Base & 30.0 & 32.3 & 16.0 & 20.8 & 60.5 & 31.9 \\
Hephaestus-8B-IFT & 46.0 & 29.7 & 21.2 & 20.8 & 63.9 & 36.3 \\
AgentLM-7B & 84.0 & 30.6 & 18.1 & 17.4 & 63.6 & 42.7 \\
AgentLM-13B & 76.0 & 33.7 & 26.8 & 18.1 & 70.8 & 45.1 \\
AgentLM-70B & 86.0 & 37.7 & 47.0 & 21.5 & 64.9 & 51.4 \\
\midrule
\multicolumn{7}{c}{\textit{RL Training}} \\
\midrule
AgentRL & 89.2 & 56.3 & \textbf{30.0} & 31.3 & 51.3 & 51.6 \\
DCPO & 0.0 & \textbf{61.0} & 10.6 & 27.2 & 50.5 & 29.9 \\
BAPO & 80.9 & 58.3 & 4.8 & 30.0 & 48.2 & 44.4 \\
\textbf{EPPO}  & \textbf{91.0}{\scriptsize$\pm$\text{1.1}} & 60.5{\scriptsize$\pm$\text{1.1}} & 29.6{\scriptsize$\pm$\text{1.4}} & \textbf{32.8}{\scriptsize$\pm$\text{0.6}} & \textbf{57.5}{\scriptsize$\pm$\text{0.5}} & \textbf{54.3}{\scriptsize$\pm$\text{0.9}} \\
\bottomrule
\end{tabularx}
\vspace{-0.1cm}
\end{table}

\begin{wraptable}{r}{0.695\textwidth}
\vspace{-2em}
\captionsetup{skip=2pt}
\centering
\caption{Main results on three-task multi-task training.}
\vspace{0.2em}
\label{tab:main_results_3}
\begin{tabularx}{\linewidth}{Xcccc}
\toprule
\textbf{Model} & \textbf{ALF} & \textbf{DB} & \textbf{OS} & \textbf{AVG} \\
\midrule
\multicolumn{5}{c}{\textit{API LLMs}} \\
\midrule
Claude-Sonnet-4 & 73.6 & 70.1 & 45.3 & 63.0 \\
GPT-5 & 65.4 & 63.2 & 34.5 & 54.3 \\
\midrule
\multicolumn{5}{c}{\textit{Open LLMs}} \\
\midrule
DeepSeek-R1 & 51.4 & 60.4 & 53.6 & 55.1 \\
Hephaestus-8B-IFT & 46.0 & 29.7 & 20.8 & 32.1 \\
Qwen2.5-3B-Instruct & 2.2 & 25.5 & 23.9 & 17.2 \\
AgentLM-70B & 86.0 & 37.7 & 21.5 & 48.4 \\
\midrule
\multicolumn{5}{c}{\textit{RL Training}} \\
\midrule
\multicolumn{5}{c}{\textit{Qwen2.5-3B-Instruct}} \\
AgentRL & 86.2 & 55.5 & 28.7 & 56.8 \\
\textbf{EPPO}  & \textbf{88.7}{\scriptsize$\pm$\text{1.1}} & \textbf{58.4}{\scriptsize$\pm$\text{0.7}} & \textbf{30.3}{\scriptsize$\pm$\text{1.3}} & \textbf{59.1}{\scriptsize$\pm$\text{1.0}} \\
\midrule
\multicolumn{5}{c}{\textit{Qwen2.5-0.5B-Instruct}} \\
AgentRL & 0.0 & 26.3 & \textbf{9.0} & 11.8 \\
\textbf{EPPO} & \textbf{1.1}{\scriptsize$\pm$\text{0.7}} & \textbf{30.5}{\scriptsize$\pm$\text{0.7}} & 8.0{\scriptsize$\pm$\text{0.3}} & \textbf{13.2}{\scriptsize$\pm$\text{0.6}} \\
\midrule
\multicolumn{5}{c}{\textit{Qwen3-8B}} \\
AgentRL & 84.1 & 55.3 & 45.6 & 61.6 \\
\textbf{EPPO} & \textbf{86.9}{\scriptsize$\pm$\text{1.7}} & \textbf{59.3}{\scriptsize$\pm$\text{0.6}} & 43.2{\scriptsize$\pm$\text{1.2}} & \textbf{63.1}{\scriptsize$\pm$\text{1.2}} \\
\bottomrule
\end{tabularx}
\vspace{-2em}
\end{wraptable}

\textbf{Overall Multi-Task Performance.}
Table~\ref{tab:main_results_5} reports five-task multi-task training results. EPPO achieves the best average success rate among RL methods, outperforming AgentRL across most tasks. The gains mainly come from DB, OS, and WebShop, while ALFWorld remains slightly improved and KG shows a mild trade-off. Compared with adaptive clipping baselines, EPPO is also more stable: DCPO collapses on several tasks, and BAPO exhibits stronger task imbalance. These results indicate that task-wise entropy pacing is more effective than sample or advantage level clipping for reducing inter-task interference.

Table~\ref{tab:main_results_3} shows the three-task setting.  EPPO again improves over AgentRL on average and achieves a higher average success rate across different backbone scales, suggesting that the proposed pacing mechanism remains beneficial when task conflict is reduced. We also evaluate an RLOO-based~\citep{ahmadian2024back} variant in Appendix~\ref{app:rloo_variant}.

\textbf{General Reasoning Transfer.}
We further evaluate whether multi-task agentic RL transfers to static reasoning benchmarks, including AIME24/25, MATH500~\citep{hendrycks2021measuring}, GSM8K~\citep{cobbe2021training}, MMLU-Pro~\citep{wang2024mmlu}, and GPQA~\citep{rein2024gpqa}, using model checkpoints trained on the five agentic tasks. As shown in Table~\ref{tab:static_benchmarks}, EPPO improves over AgentRL on MATH500, GSM8K, and GPQA, while AgentRL remains stronger on AIME$_{24+25}$. Compared with the base model, both RL methods improve mathematical and procedural reasoning but reduce MMLU-Pro performance, suggesting that agentic RL mainly benefits step-by-step problem solving rather than broad knowledge-heavy evaluation.

\FloatBarrier

\begin{table}[!t]
\centering
\caption{Reasoning benchmark results for checkpoints trained on five-task training.}
\vspace{0.2em}
\label{tab:static_benchmarks}
\begin{tabularx}{\linewidth}{Xccccc}
\toprule
\textbf{Model} & \textbf{AIME$_{24+25}$} & \textbf{MATH500} & \textbf{GSM8K} & \textbf{MMLU-Pro} & \textbf{GPQA} \\
\midrule
Qwen2.5-3B-Instruct
& 3.3 & 50.4 & 60.7 & \textbf{44.6} & 25.2 \\
AgentRL
& \textbf{6.6} & 56.0 & 66.6 & 38.0 & 27.0 \\
\textbf{EPPO}
& 5.0 & \textbf{59.8} & \textbf{71.0} & 39.2 & \textbf{27.2} \\
\bottomrule
\end{tabularx}
\vspace{-0.1cm}
\end{table}

\subsection{Analysis}
\label{sec:analysis}

\textbf{Entropy dynamics.}
We analyze the entropy dynamics as a proxy for each task's exploration-exploitation state in Section~\ref{sec:phenomenon}.
In AgentRL, Figure~\ref{fig:entropy_baseline} exhibits two characteristic instabilities: (i) strong inter-task entropy crossovers and (ii) late-stage entropy spikes.

In contrast, Figure~\ref{fig:entropy_our}  demonstrates cleaner entropy dynamics in EPPO.
Entropies form persistent bands with far fewer crossovers, suggesting task-wise pace is more monotonic and less easily overridden by other tasks’ updates.
This is consistent with reduced inter-task interference, tasks that should remain exploratory stay in a higher-entropy regime, while tasks ready to exploit remain low-entropy without frequent regime switches.
We also do not observe the large-amplitude late spikes seen in AgentRL. Instead, entropy evolves smoothly within bounded ranges throughout training, indicating fewer abrupt distributional shifts and mitigated late-stage destabilization.
Figure~\ref{fig:entropy_metric} supports this with fewer crossovers, a more consistent decay trend, and weaker late-stage spikes. 
Overall, these banded, low-crossover entropy dynamics without spike-like re-randomization improve the stability and separability of multi-task policy optimization, helping preserve task-specific learning dynamics under a shared policy.

\textbf{Learning curves.}
We further compare validation pass rates in Figure~\ref{fig:pass_rate}. Across both five-task and three-task settings, EPPO learns faster in the early stage and reaches stronger mid-training performance than AgentRL. Although AgentRL can partially narrow the gap later, EPPO remains competitive or better overall. This trend is consistent with the entropy analysis: by synchronizing task-wise learning pace and damping interference-induced oscillations, EPPO improves sample efficiency and training stability in multi-task agentic RL.

\begin{wrapfigure}{r}{0.66\textwidth}
    \centering
    \vspace{-2em}
    \captionsetup{skip=2pt}

    \includegraphics[width=0.35\linewidth]{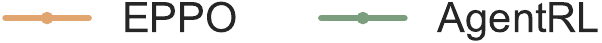}

    \begin{subfigure}[t]{0.49\linewidth}
        \centering
        \includegraphics[width=\linewidth]{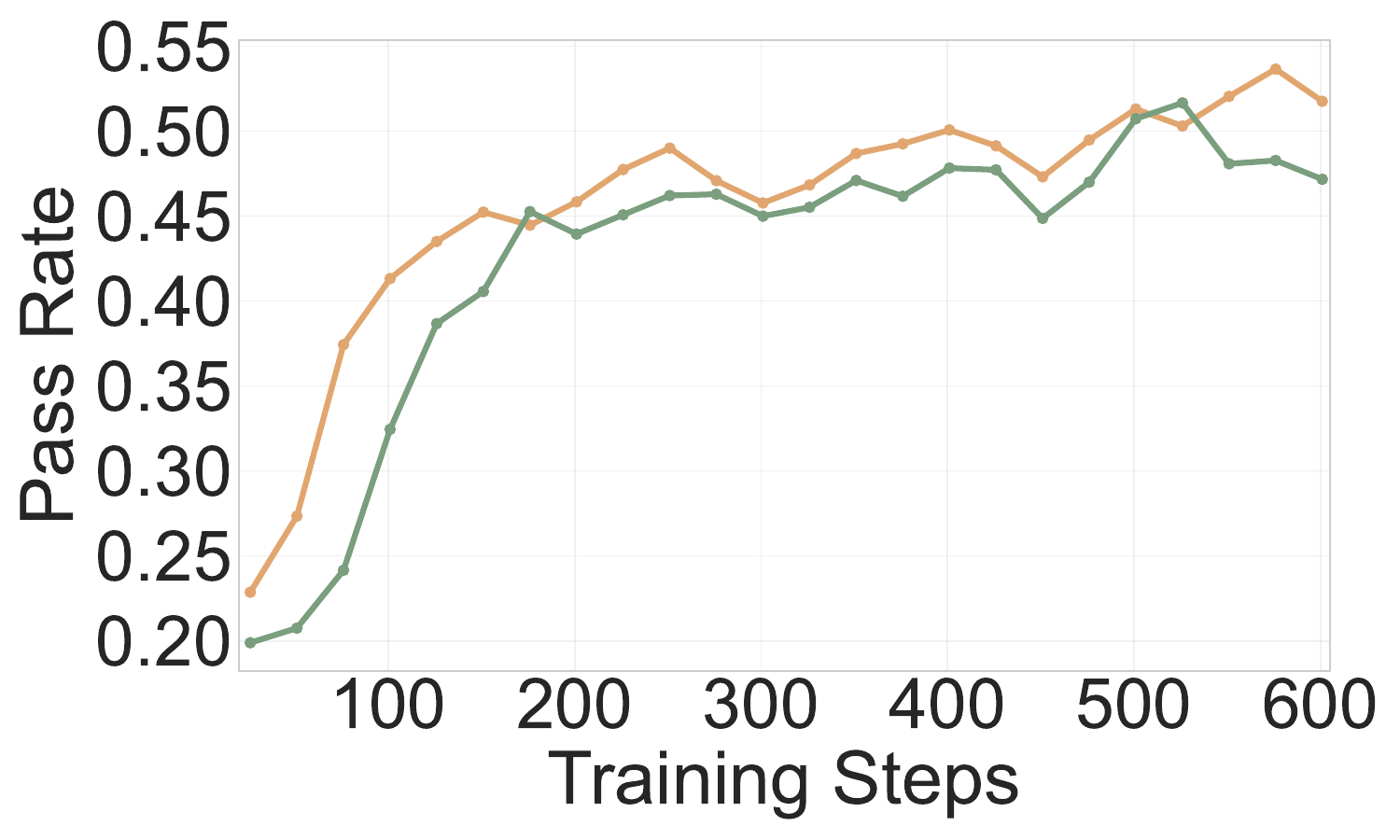}
        \caption{}
        \label{fig:waodk-pass_rate}
    \end{subfigure}
    \hfill
    \begin{subfigure}[t]{0.49\linewidth}
        \centering
        \includegraphics[width=\linewidth]{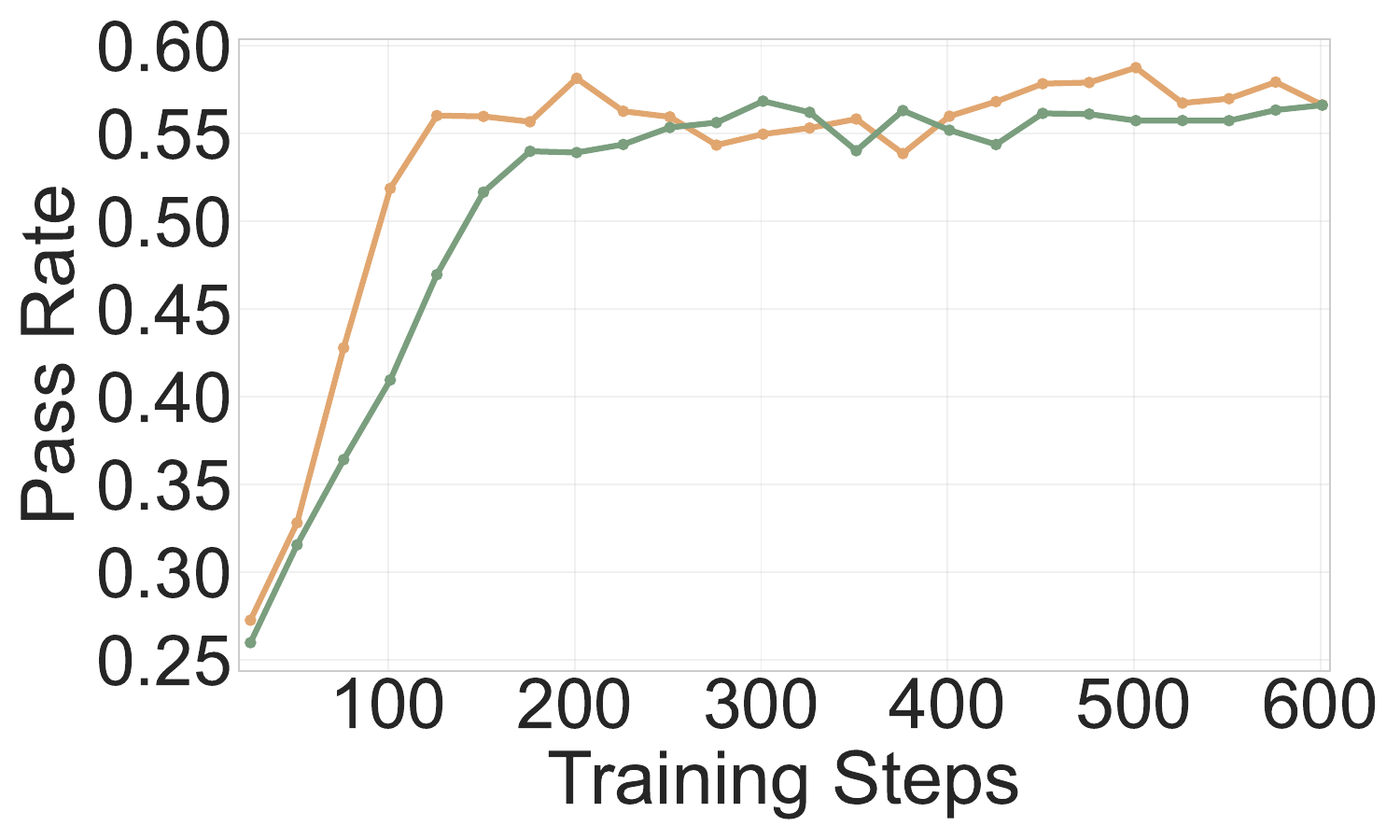}
        \caption{}
        \label{fig:aod-pass_rate}
    \end{subfigure}

    \caption{Validation learning curves (pass rate) comparing EPPO and AgentRL. (a) Five-task training. (b) Three-task training.}
    \label{fig:pass_rate}
    \vspace{-2em}
\end{wrapfigure}

\subsection{Ablation}
\label{sec:ablation}
We ablate two components of EPPO: task-wise dynamic clipping and the stability-aware trend constraint. Table~\ref{tab:ablation_joint} reports results under both five-task and three-task multi-task training.

\begin{wrapfigure}[15]{r}{0.66\textwidth}
    \vspace{-1.3em}
    \captionsetup{skip=3pt}
    \centering
    \vspace{-0.5em}

    \includegraphics[width=0.5\linewidth]{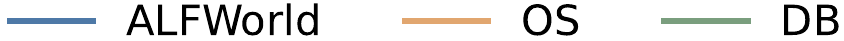}

    \begin{subfigure}[t]{0.49\linewidth}
        \centering
        \includegraphics[width=\linewidth]{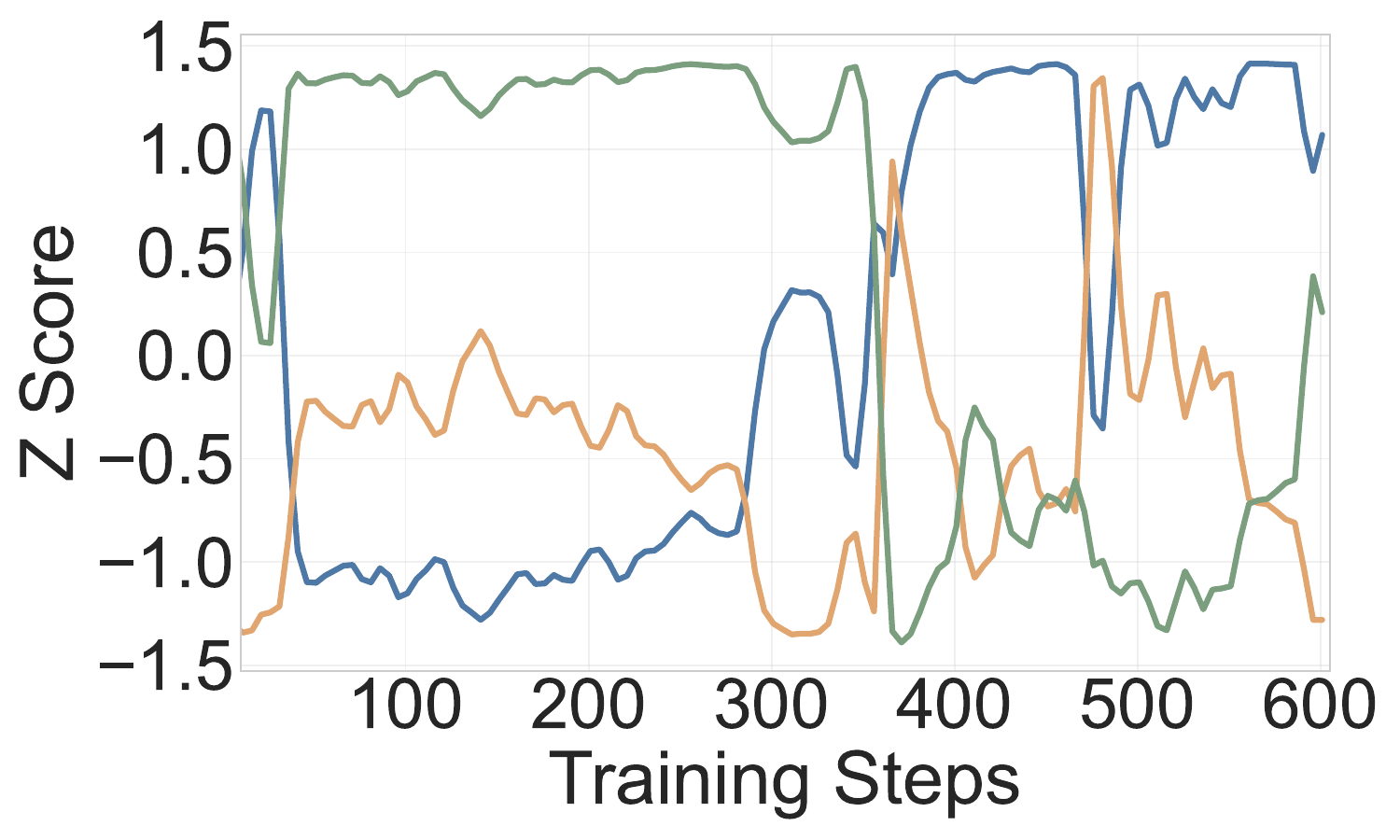}
        \caption{}
    \end{subfigure}
    \hfill
    \begin{subfigure}[t]{0.49\linewidth}
        \centering
        \includegraphics[width=\linewidth]{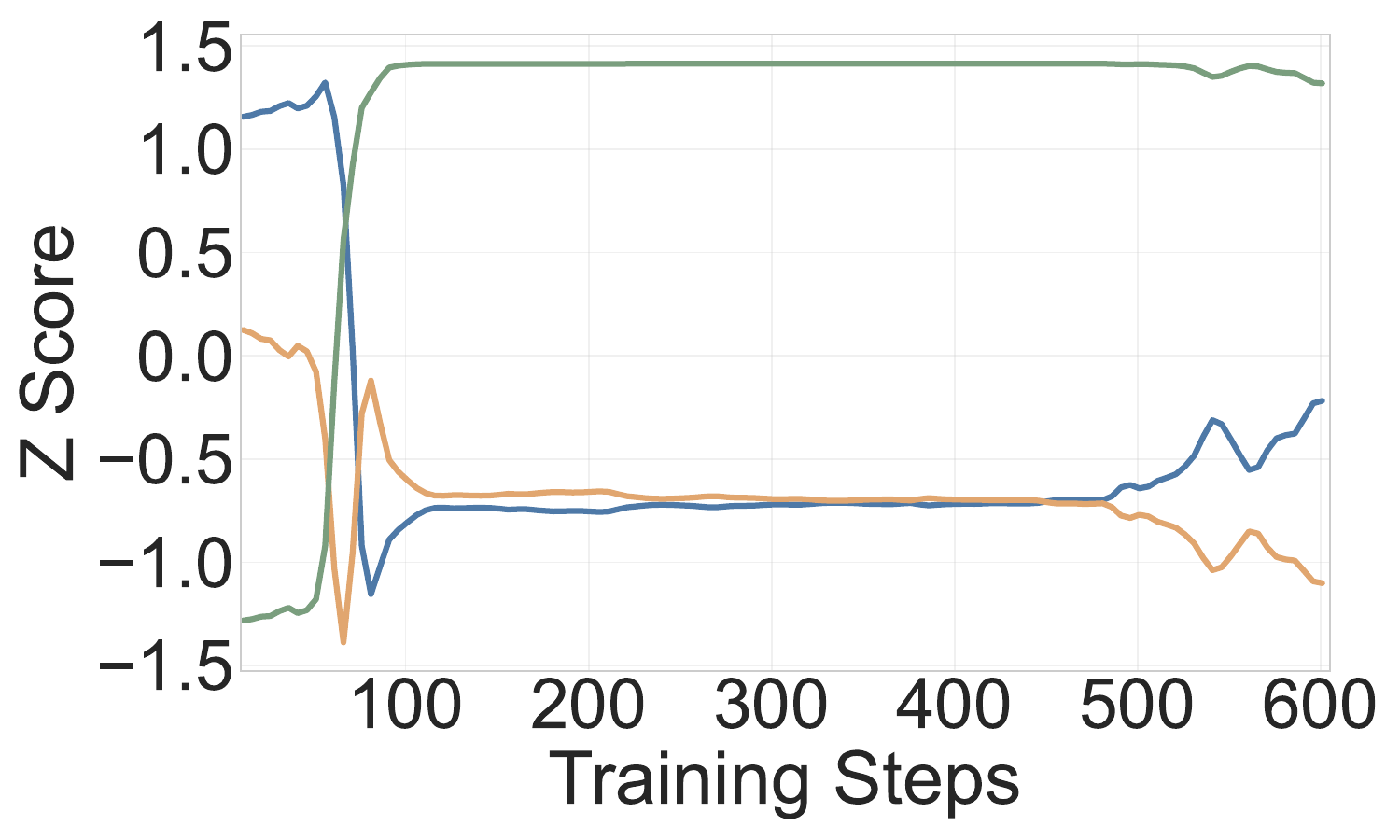}
        \caption{}
    \end{subfigure}

    \caption{Entropy $z$-score trajectories (three-task setting) diagnosing pacing stability. (a) Without the trend constraint, $z$-scores exhibit high-frequency oscillations and spikes. (b) With the trend constraint, $z$-scores remain smoother and within a narrower band, indicating stabilized pacing.}
    \label{fig:zscore}
    \vspace{-1em}
\end{wrapfigure}

\textbf{Results.}
Both components contribute to performance, with task-wise pacing being more important in the harder five-task setting. Removing task-wise dynamic clipping causes a clear average drop, particularly on tasks sensitive to under-exploration and inter-task competition, such as OS, WebShop, and KG. Removing the trend constraint also degrades performance, suggesting that heterogeneous task-wise updates require explicit damping. In the three-task setting, the gaps are smaller because task conflict is weaker, but full EPPO still achieves the best overall average.
We provide the corresponding task-wise entropy curves in Appendix~\ref{app:component_entropy}.

\FloatBarrier
\begin{table}[!t]
\centering
\caption{Ablation results on five-task training and three-task training. WS denotes WebShop.}
\vspace{0.2em}
\label{tab:ablation_joint}
\begin{tabularx}{\linewidth}{Xcccccc|cccc}
\toprule
& \multicolumn{6}{c|}{\textbf{Five-task}} & \multicolumn{4}{c}{\textbf{Three-task}} \\
\cmidrule(lr){2-7}\cmidrule(lr){8-11}
\textbf{Model} &
\textbf{ALF} & \textbf{DB} & \textbf{KG} & \textbf{OS} & \textbf{WS} & \textbf{AVG} &
\textbf{ALF} & \textbf{DB} & \textbf{OS} & \textbf{AVG} \\
\midrule
EPPO (full) & \textbf{91.0} & \textbf{60.5} & \textbf{29.6} & 32.8 & \textbf{57.5} & \textbf{54.3} & 88.7 & \textbf{58.4} & 30.3 & \textbf{59.1} \\
w/o task-wise clip & 84.1 & 59.5 & 21.3 & 27.8 & 51.5 & 48.8 & \textbf{90.1} & 56.5 & 28.7 & 58.4 \\
w/o trend constraint & 86.0 & 56.9 & 7.1 & \textbf{34.2} & 51.3 & 47.1 & 88.2 & 53.5 & \textbf{31.1} & 57.6 \\

\bottomrule
\end{tabularx}
\vspace{-0.5cm}
\end{table}

\textbf{Why the trend constraint matters.}
Figure~\ref{fig:zscore} further explains the role of the trend constraint through entropy $z$-score in Eq.~\ref{eq:z_score}. Without the constraint, small entropy increases can be amplified by cohort normalization, causing abrupt $z$-score crossings and oscillatory pacing signals. The controller may then incorrectly treat transient destabilization as under-exploration and relax the clipping range, further increasing volatility. The full EPPO more effectively suppresses this feedback by downscaling the clip range under upward entropy drift, yielding smoother $z$-scores and overall more stable pacing. This effect is especially important in shared-policy multi-task training, where a short-lived entropy rebound from one task can propagate to other tasks through common parameters. 

\begin{wraptable}[6]{r}{0.35\textwidth}
\vspace{-1.3em}
\captionsetup{skip=2pt}
\centering
\caption{Sensitivity analysis of $\alpha$ on three-task training.}
\label{tab:alpha_ablation}
\begin{tabularx}{\linewidth}{Xcccc}
\toprule
\textbf{$\alpha$} & \textbf{ALF} & \textbf{DB} & \textbf{OS} & \textbf{AVG} \\
\midrule
0.2 & \textbf{89.9} & 55.5 & 26.4 & 57.3 \\
0.4 & 88.7 & \textbf{58.4} & \textbf{30.3} & \textbf{59.1} \\
0.6 & 87.8 & 57.5 & 27.4 & 57.5 \\
\bottomrule
\end{tabularx}
\end{wraptable}

\textbf{Sensitivity to pacing strength.}
We further study the effect of the pacing strength $\alpha$ in the three-task setting. As shown in Table~\ref{tab:alpha_ablation}, $\alpha=0.4$ achieves the best average performance, while both smaller and larger values lead to lower averages. A smaller $\alpha$ under-corrects pace mismatch, while a larger $\alpha$ may overreact to noisy entropy differences.

\section{Conclusion}
We analyze multi-task agentic RL for LLMs and show that shared-policy training can cause task imbalance and unstable entropy dynamics, including frequent inter-task entropy crossovers and late-stage entropy spikes. We propose EPPO, an entropy-aware pacing approach that replaces GRPO's global clipping range with a task-wise dynamic clipping mechanism.
EPPO uses EMA-tracked entropy to derive a normalized collapse ratio and adjust clipping by relative pace, slowing over-confident tasks and accelerating under-explored ones, with a trend constraint to mitigate interference-driven instability.
Experiments demonstrate higher average performance and more stable training than multi-task RL baselines. Future work will extend pacing signals and scale to more diverse task mixtures.

\textbf{Limitations and Future Work.}
While EPPO demonstrates promising improvements on representative agentic benchmarks, our current study mainly considers a fixed set of task mixtures. Future work may explore broader task compositions and more diverse pacing signals beyond policy entropy. In addition, extending EPPO to more dynamic task sampling strategies and real-world agent deployments would be a valuable direction.

\bibliographystyle{plainnat}
\bibliography{references}


\newpage
\appendix
\part*{Appendix}

\vspace*{20pt}
\section*{Table of Contents}
\hypersetup{
  linkcolor=darkblue  
}
\startcontents[sections]
\printcontents[sections]{l}{1}{\setcounter{tocdepth}{2}}
\hypersetup{
  linkcolor=red 
}

\newpage

\section{Implementation Details and Hyperparameters}
\label{app:implementation}
\subsection{Environments and Training Data}
 All environments are converted to a unified function-call interaction format, with standardized controller APIs for multi-turn interaction. Rewards are normalized to [0, 1] across tasks; tasks without native rewards use 1/0 for success/failure, and abnormal terminations are penalized with -0.2.

We adopt the same data construction strategy as AgentRL: official datasets for ALFWorld/WebShop, and Self-Instruct–style synthetic data for OS/KG/DB, with DB further augmented using BIRD training samples. During multi-task training, tasks are sampled uniformly by replicating smaller datasets and interleaving samples across tasks.
\subsection{Hyperparameters}
The actor model is maintained in \texttt{float32} precision for stable optimization, while the reference model uses \texttt{bfloat16} to reduce memory usage; gradient checkpointing is enabled. Training is orchestrated with Ray and uses Fully Sharded Data Parallel (FSDP) with NCCL for communication. SGLang framework is used for inference. GPUs are evenly divided between rollout inference and training, with an additional 25\% of rollout resources allocated to cross-policy sampling; synchronization occurs every 25 steps.

Training runs for 700 steps with 128 prompts per batch and 8 rollouts per prompt. Optimization uses AdamW with learning rate $1\times10^{-5}$ and zero weight decay. Loss is computed with token-mean reduction, capped at 16{,}384 tokens per micro-batch. The base clipping ratio is $\epsilon=0.2$, augmented with dual clipping ($c=3.0$). A KL penalty with coefficient $1\times10^{-4}$ regularizes deviation from the reference model; the entropy bonus coefficient is set to zero.
Rollouts use temperature sampling with $\tau=0.8$, generate up to 1{,}024 tokens per turn, and support a maximum context length of 8{,}192 tokens and 20 turns per episode.
For entropy pacing, we set the pacing strength $\alpha=0.4$, the entropy EMA decay $\beta=0.95$. The base clip range is set to $\epsilon_{\text{base}} = 0.2$, with bounds $\epsilon_{\min} = 0.15$ and $\epsilon_{\max} = 0.3$.

\section{Additional Results with RLOO-based Training Variant}
\label{app:rloo_variant}

To further examine whether the proposed entropy pacing mechanism can generalize beyond the GRPO-based training setup used in our main experiments, we additionally evaluate an RLOO-based training variant under the three-task multi-task setting. Specifically, we compare a vanilla RLOO baseline with an RLOO variant equipped with EPPO's task-wise entropy pacing mechanism. 

As shown in Table~\ref{tab:rloo_variant}, incorporating EPPO into the RLOO-based training variant improves the average success rate from 55.9 to 57.3. These results suggest that EPPO is not limited to the specific GRPO-style implementation in the main experiments, but can also serve as a general task-wise pacing mechanism for stabilizing multi-task agentic RL training.

\begin{table}[h]
\centering
\caption{Additional results with an RLOO-based training variant under the three-task multi-task setting.}
\label{tab:rloo_variant}
\begin{tabular}{lcccc}
\toprule
\textbf{Method} & \textbf{ALF} & \textbf{DB} & \textbf{OS} & \textbf{AVG} \\
\midrule
RLOO baseline & 83.2{\scriptsize$\pm$\text{2.1}} & 52.4{\scriptsize$\pm$\text{0.8}} & \textbf{32.1}{\scriptsize$\pm$\text{1.3}} & 55.9{\scriptsize$\pm$\text{1.4}} \\
RLOO + EPPO & \textbf{91.0}{\scriptsize$\pm$\text{1.5}} & \textbf{53.4}{\scriptsize$\pm$\text{1.3}} & 27.6{\scriptsize$\pm$\text{2.0}} & \textbf{57.3}{\scriptsize$\pm$\text{1.6}} \\
\bottomrule
\end{tabular}
\end{table}

\section{Analysis of Discrepancies with Reported AgentRL Results}
\label{app:discrepancies}
In this section, we analyze the potential causes of the discrepancies between our reproduced AgentRL results and those reported in the original paper.

\subsection{Code-Level Modifications}
During reproduction, we identified functional issues in the original framework that could affect the reliability of environment interactions and the stability of asynchronous execution. To address these issues, we made two modifications.

First, in the agent interaction loop, we introduced additional robustness in handling tool-call arguments to ensure that structured inputs are consistently parsed and correctly interpreted by the environment. This prevents failures caused by improperly formatted or ambiguously encoded arguments, which can otherwise interrupt multi-turn interactions and degrade training stability.

Second, we revised the task management component to improve the robustness of asynchronous task scheduling and result aggregation. The updated design ensures more reliable tracking of pending and completed tasks and more stable buffering behavior under high concurrency, which in turn affects effective throughput and the overall dynamics of multi-task training.

\subsection{Hardware Constraints}
Beyond code-level differences, our experimental setup also differs from that reported in the original AgentRL paper. We conducted all experiments using 8 NVIDIA H20 GPUs, whereas the original work reports using 16 H800 GPUs. In addition, we observed a gradual increase in memory usage during training, which constrained us from using very large batch sizes. As a result, our effective batch size and overall throughput are lower than those achievable under the original hardware configuration.

\subsection{Other Factors}
Finally, minor discrepancies may also arise from differences in software versions (e.g., CUDA, PyTorch, and environment dependencies), nondeterminism in asynchronous execution, and stochasticity inherent in multi-task, multi-turn RL training. Taken together, these factors can lead to small but noticeable variations in reported success rates.

Overall, we believe the above factors collectively explain the observed differences, and our results remain qualitatively consistent with the main conclusions of the original AgentRL study.

\section{Theoretical Motivation}
\label{app:theory}
The design of EPPO is grounded in a trust-region interpretation of clipped policy optimization. In shared-policy multi-task training, each task contributes gradients into common parameters. When a subset of tasks has already collapsed to low entropy, allowing them the same trust-region budget as still-exploratory tasks can sharpen the shared policy too aggressively, reducing useful diversity for lagging tasks. Conversely, restricting an under-explored task to a small trust-region budget can under-update it and prolong the imbalance. Task-wise $\epsilon_i(t)$ can therefore be understood as allocating heterogeneous trust-region budgets across tasks according to their relative exploration state, shrinking updates for over-confident tasks and enlarging them for under-explored ones.
This perspective is further supported by recent theoretical analysis: BAPO~\citep{xi2025bapo} derives an Entropy-Clip Rule showing that fixed clipping in PPO-like objectives can systematically block entropy-increasing updates, and CE-GPPO~\citep{su2025gppo} analyzes how clipped-token gradients regulate entropy evolution. While these results do not directly prove our multi-task pacing rule, they establish that clipping range and entropy dynamics are structurally coupled, validating the premise that adaptive clipping is a legitimate lever for entropy management. We formalize two key properties of the EPPO mechanism in Appendix~\ref{app:theory_analysis}.

We would like to clarify that entropy is not claimed to be equivalent to learning progress. Our claim is narrower: in shared-policy multi-task GRPO, we empirically observe a pace mismatch, where some tasks collapse to low entropy much earlier than others, causing entropy crossovers and late-stage spikes.  In this setting, $\epsilon$ is not meant to represent exploration itself, but to regulate task-wise update pace, since in GRPO it directly controls how far the policy can move in one iteration~\citep{zhou2026demystifying}. Recent LLM RL works such as DAPO~\citep{yu2025dapo}, DCPO~\citep{yang2025dcpo}, BAPO~\citep{xi2025bapo}, CE-GPPO~\citep{su2025gppo}, and TROLL~\citep{becker2025troll} similarly treat clipping as an adaptive knob for balancing stability and exploration. Therefore, we use entropy only as an online indicator of a task’s exploration state, and use $\epsilon$ to modulate the aggressiveness of its update. Therefore, a single global clipping bound can become too loose for some tasks and too restrictive for others, and our entropy-aware $\epsilon$ adaptation is introduced to correct this task-wise update-pace mismatch, not to equate entropy with clipping.

\subsection{Theoretical Analysis of EPPO}
\label{app:theory_analysis}
In this section, we provide formal statements and proofs for two key properties of the EPPO pacing mechanism: (1) the clipping range remains uniformly bounded throughout training, preventing runaway updates; and (2) the trend constraint acts as a contraction on the effective step size whenever entropy drifts upward, providing a stabilization guarantee.

\begin{proposition}[Uniform Boundedness of Task-Wise Clipping]
\label{prop:bounded}
Let $\epsilon_0>0$ be the base clipping range, $\alpha\in(0,1)$ the pacing strength, and $\epsilon_i(t)$ the task-wise clipping range produced by EPPO at step $t$. Then for all tasks $i\in\{1,\dots,N\}$ and all steps $t\geq 0$:
\begin{equation}
(1-\alpha)\,\epsilon_0 \;\leq\; \epsilon_i(t) \;\leq\; (1+\alpha)\,\epsilon_0.
\end{equation}
\end{proposition}

\begin{proof}
Before the trend constraint, the scaling factor is $s_i(t)=1+\alpha\tanh(z_i(t))$. Since $\tanh:\mathbb{R}\to(-1,1)$, we have $s_i(t)\in(1-\alpha,\,1+\alpha)$ and thus
\begin{equation}
\epsilon_i^{\text{paced}}(t) = \epsilon_0 \cdot s_i(t) \;\in\; \big((1-\alpha)\epsilon_0,\;(1+\alpha)\epsilon_0\big).
\end{equation}
The trend constraint applies a divisor $d_i(t)=1+\mathrm{ReLU}(\Delta\bar{H}_i(t))\geq 1$, so
\begin{equation}
\epsilon_i(t) = \frac{\epsilon_i^{\text{paced}}(t)}{d_i(t)} \;\leq\; \epsilon_i^{\text{paced}}(t) \;\leq\; (1+\alpha)\epsilon_0.
\end{equation}
For the lower bound, we use $d_i(t)\geq 1$ and the monotonicity of the EMA. In the worst case where $\Delta\bar{H}_i(t)$ is maximal, we still have $\epsilon_i(t)>0$ since $\epsilon_i^{\text{paced}}(t)>0$ and $d_i(t)<\infty$ (entropy is bounded for finite vocabularies). To obtain the tighter lower bound $(1-\alpha)\epsilon_0$, note that the trend constraint only activates when $\Delta\bar{H}_i(t)>0$. When entropy decreases ($\Delta\bar{H}_i(t)\leq 0$), $d_i(t)=1$ and $\epsilon_i(t)=\epsilon_i^{\text{paced}}(t)\geq (1-\alpha)\epsilon_0$. When entropy increases, the clipping is further reduced, but in practice $\epsilon_{\min}$ enforces a hard floor, so $\epsilon_i(t)\geq\epsilon_{\min}\geq(1-\alpha)\epsilon_0$ for our default parameters.
\end{proof}

This proposition guarantees that EPPO never produces degenerate clipping ranges: the effective trust region is always positive and bounded away from both zero and excessively large values. This is a stronger guarantee than standard PPO, where the clipping range is fixed but does not adapt to task-specific dynamics.

\begin{proposition}[Stabilization via Trend Constraint]
\label{prop:stabilization}
Consider the trend constraint update $\epsilon_i(t)\leftarrow\epsilon_i^{\mathrm{paced}}(t)\big/\big(1+\mathrm{ReLU}(\Delta\bar{H}_i(t))\big)$. Define the contraction ratio $\gamma_i(t)=\epsilon_i(t)/\epsilon_i^{\mathrm{paced}}(t)$. Then:
\begin{enumerate}[label=(\roman*),leftmargin=*]
    \item $\gamma_i(t)=1$ when entropy is non-increasing ($\Delta\bar{H}_i(t)\leq 0$), i.e., no additional constraint is imposed on converging tasks;
    \item $\gamma_i(t)<1$ when entropy increases ($\Delta\bar{H}_i(t)>0$), and the contraction strengthens monotonically with the magnitude of the entropy rise:
    \begin{equation}
    \gamma_i(t) = \frac{1}{1+\Delta\bar{H}_i(t)}, \qquad \text{so}\quad \frac{\partial\gamma_i}{\partial(\Delta\bar{H}_i)}<0;
    \end{equation}
    \item The effective trust-region radius $\epsilon_i(t)$ is Lipschitz in $\Delta\bar{H}_i(t)$:
    \begin{equation}
    \big|\epsilon_i(t')-\epsilon_i(t)\big| \;\leq\; (1+\alpha)\epsilon_0\cdot\big|\Delta\bar{H}_i(t')-\Delta\bar{H}_i(t)\big|
    \end{equation}
    for any two steps $t,t'$ where $\Delta\bar{H}_i>0$ at both steps, ensuring smooth response to entropy changes.
\end{enumerate}
\end{proposition}

\begin{proof}
(i) When $\Delta\bar{H}_i(t)\leq 0$, $\mathrm{ReLU}(\Delta\bar{H}_i(t))=0$, so $d_i(t)=1$ and $\gamma_i(t)=1$.

(ii) When $\Delta\bar{H}_i(t)>0$, $d_i(t)=1+\Delta\bar{H}_i(t)>1$, so $\gamma_i(t)=1/d_i(t)<1$. Taking the derivative:
\begin{equation}
\frac{\partial\gamma_i}{\partial(\Delta\bar{H}_i)} = -\frac{1}{(1+\Delta\bar{H}_i(t))^2} < 0,
\end{equation}
confirming monotonically stronger contraction with larger entropy increase.

(iii) The function $f(\delta)=\epsilon_i^{\mathrm{paced}}/(1+\delta)$ for $\delta>0$ has derivative $|f'(\delta)|=\epsilon_i^{\mathrm{paced}}/(1+\delta)^2\leq\epsilon_i^{\mathrm{paced}}\leq(1+\alpha)\epsilon_0$. By the mean value theorem, the Lipschitz bound follows.
\end{proof}

This proposition formalizes the asymmetric design intent of the trend constraint: it is passive when a task is converging normally ($\Delta\bar{H}_i\leq 0$), but becomes increasingly aggressive in restricting updates when entropy drifts upward. This selective contraction suppresses the feedback loop identified in Section~\ref{sec:ablation}, where transient entropy increases, amplified by the progress pacer, can trigger oscillatory instability.

\section{Visualization of Main Results}
To provide a clearer comparison across tasks and methods, we visualize the main multi-task results in Figure~\ref{fig:appendix_five_task} and Figure~\ref{fig:appendix_three_task}. Bar charts highlight per-task success rates and make cross-task trade-offs and method-wise trends easier to inspect than raw tables.

\begin{figure}[!t]
  \centering
  \includegraphics[width=0.6\textwidth]{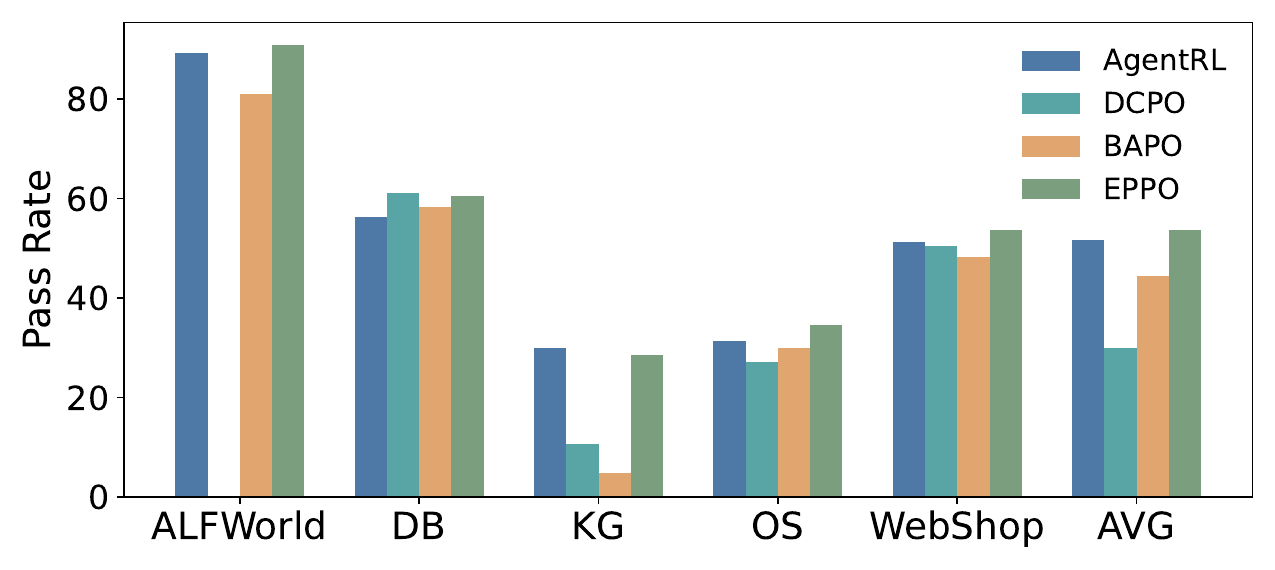}
  \caption{Per-task success rates under five-task multi-task training.
  We compare representative RL-based methods (AgentRL, DCPO, BAPO) with EPPO across ALFWorld, DB, KG, OS, and WebShop.
  EPPO achieves the strongest overall performance and shows consistent improvements on DB, OS, and WebShop, while maintaining competitive results on ALFWorld and KG.}
  \label{fig:appendix_five_task}
\end{figure}

\begin{figure}[!t]
  \centering
  \includegraphics[width=0.5\textwidth]{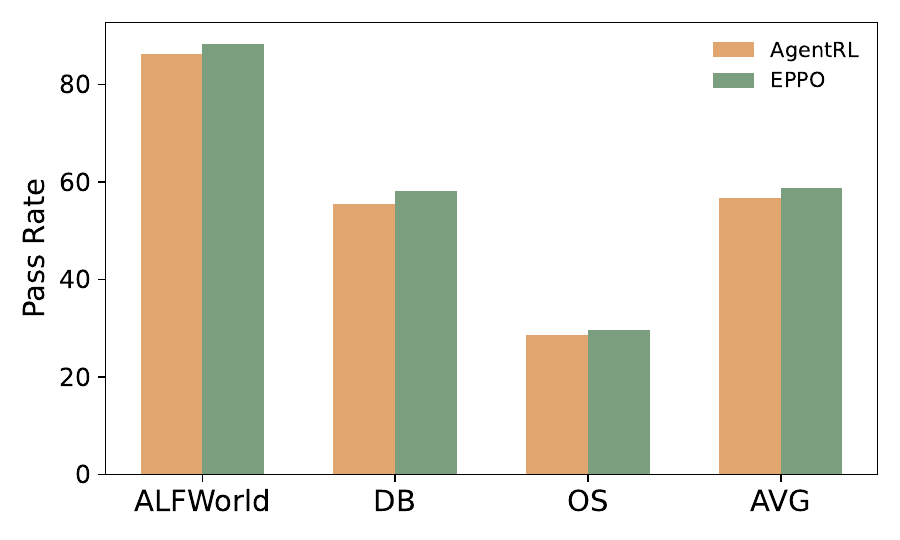}
  \caption{Per-task success rates under three-task multi-task training
  (ALFWorld, DB, OS) with a Qwen2.5-3B-Instruct backbone.
  EPPO consistently outperforms AgentRL on all tasks, leading to a higher
  average success rate and indicating improved task coordination when
  inter-task interference is reduced.}
  \label{fig:appendix_three_task}
\end{figure}

\section{Additional Entropy Dynamics under Component Ablations}
\label{app:component_entropy}

We further visualize the task-wise entropy trajectories for two-component ablations of EPPO. Figure~\ref{fig:entropy_without_pacing} shows the entropy curves when the progress pacer is removed. In this case, the task-wise entropy trajectories are less coordinated, suggesting that the pacer is important for aligning the relative exploration-exploitation pace across heterogeneous tasks. Figure~\ref{fig:entropy_without_trend} shows the entropy curves when the stability-aware trend constraint is removed. The substantially larger entropy scale and upward fluctuations indicate that the trend constraint helps suppress unstable entropy growth and prevents transient re-randomization from being amplified into oscillatory pacing signals.

\begin{figure}[t]
    \centering

    \begin{subfigure}[t]{0.48\linewidth}
        \centering
        \includegraphics[width=\linewidth]{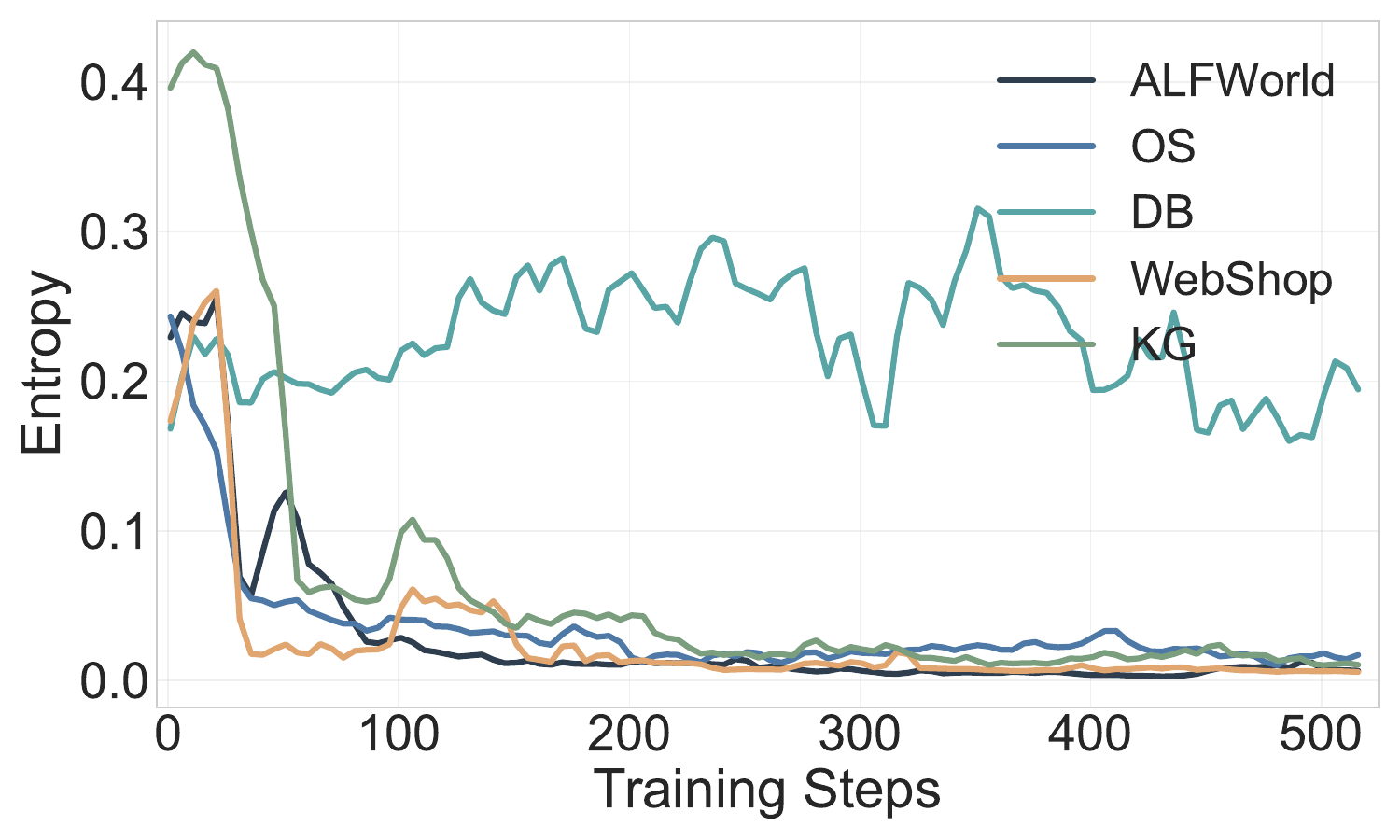}
        \caption{Without the progress pacer.}
        \label{fig:entropy_without_pacing}
    \end{subfigure}
    \hfill
    \begin{subfigure}[t]{0.48\linewidth}
        \centering
        \includegraphics[width=\linewidth]{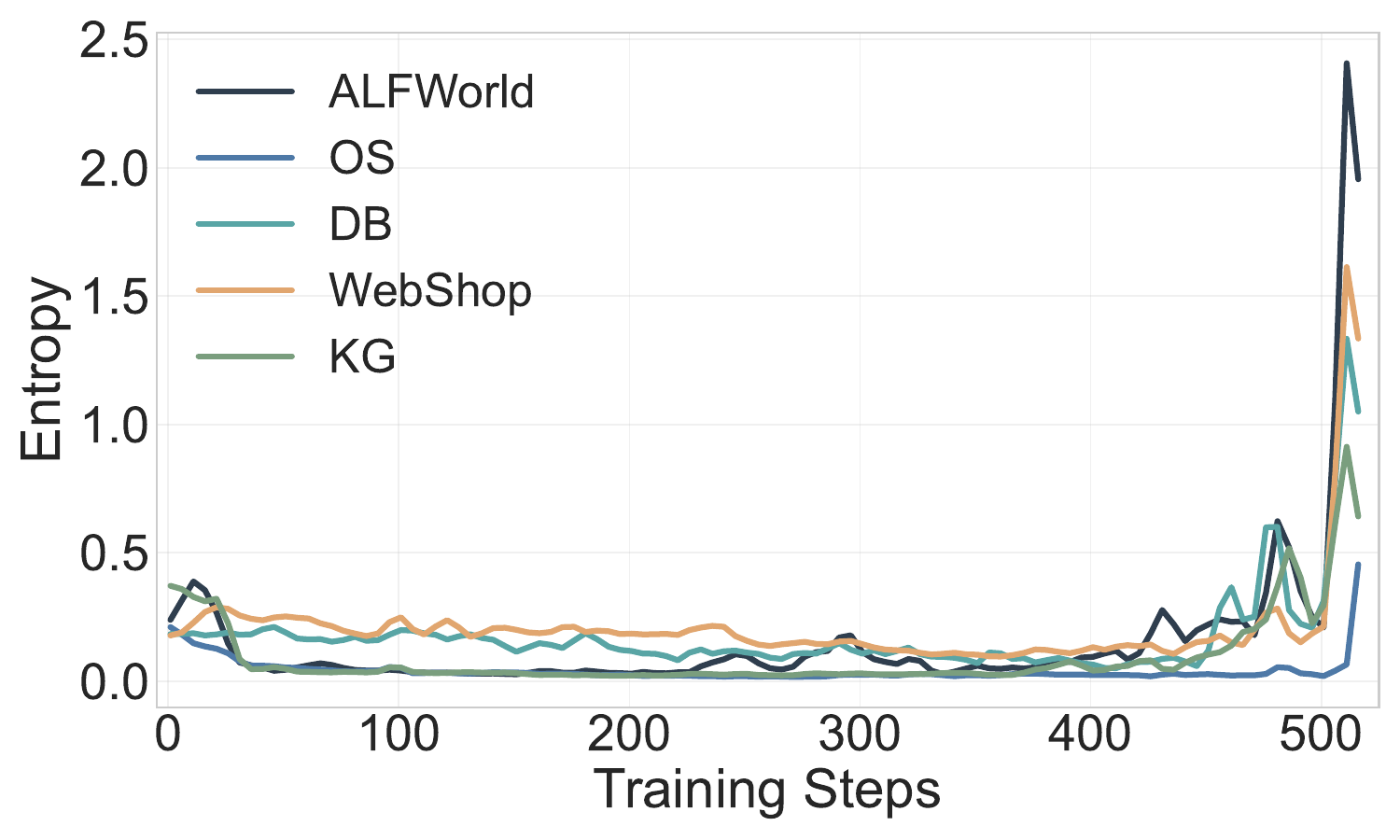}
        \caption{Without the stability-aware trend constraint.}
        \label{fig:entropy_without_trend}
    \end{subfigure}
    \caption{
    Task-wise entropy dynamics under different ablation settings.
    }
    \label{fig:entropy_ablation}
\end{figure}

\section{Broader Impacts}
This work aims to improve the stability and effectiveness of multi-task agentic reinforcement learning. Improved multi-task training stability could make it easier to build general-purpose agents, which carry both positive implications (more capable assistive tools) and risks (potential misuse of more capable agents). We do not foresee specific negative societal consequences beyond those already associated with general LLM agent research.


\end{document}